\documentclass{article}
\usepackage{hyphenat}
\usepackage[preprint,nonatbib]{neurips_2025}

\usepackage[utf8]{inputenc}
\usepackage[T1]{fontenc}
\usepackage[numbers,sort&compress]{natbib}
\usepackage{amsmath,amssymb,amsthm,mathtools}
\usepackage{graphicx}
\usepackage{booktabs}
\usepackage{microtype}
\usepackage{xcolor}
\usepackage{url}
\definecolor{NavyBlue}{HTML}{1A4E8A}
\usepackage[colorlinks=true,linkcolor=NavyBlue,citecolor=NavyBlue,
            urlcolor=NavyBlue,breaklinks=true]{hyperref}
\hypersetup{
  pdftitle={The Distributional View of Knowledge Distillation},
  pdfauthor={Gordei Verbii, Juho Lee},
  pdfkeywords={knowledge distillation, optimal transport, Wasserstein
               barycenter, Sinkhorn divergence, language models}}

\theoremstyle{plain}
\newtheorem{proposition}{Proposition}
\theoremstyle{definition}
\newtheorem{hypothesis}{Hypothesis}
\theoremstyle{remark}
\newtheorem{remark}{Remark}

\newcommand{\Vv}{\mathcal V}
\newcommand{\Ss}{\mathcal S}
\newcommand{\OT}{\mathrm{OT}}
\newcommand{\KL}{\mathrm{KL}}
\newcommand{\JS}{\mathrm{JS}}
\newcommand{\PPL}{\mathrm{PPL}}
\newcommand{\Deff}{D_{\mathrm{eff}}}
\newcommand{\eps}{\varepsilon}
\DeclareMathOperator*{\argmin}{arg\,min}
\DeclareMathOperator*{\argmax}{arg\,max}
\newcommand{\BaryOT}{\mbox{BaryOT}}
\newcommand{\sd}[1]{\,{\scriptsize$\pm$#1}}

\title{The Distributional View of Knowledge Distillation}

\author{%
  Gordei Verbii\\
  Independent Researcher\\
  \texttt{scigverbii@gmail.com}
  \And
  Juho Lee\\
  KAIST\\
  \texttt{juholee@kaist.ac.kr}
}

\begin{document}
\maketitle

\begin{abstract}
Token-level knowledge distillation (KD) matches two conditional distributions per position, yet the standard objectives compare them \emph{pointwise}: a Kullback-Leibler gradient is blind to \emph{which} wrong token receives probability mass. We develop a distributional view in which the teacher is represented not by a single softened output but by a family of \emph{multi-temperature views} -- marginals of the annealing path of its logits -- and the student is trained against a geometry-aware aggregate of these views under an embedding-based ground cost. We formalize the resulting design space (mixtures, log-linear pooling, entropic Wasserstein barycenters, and a debiased Sinkhorn-divergence flagship in \emph{hub} and \emph{path} forms), prove an exact collapse result showing log-linear pooling of tempered views is equivalent to a single temperature, and give a multi-marginal Schr\"odinger-bridge reading that yields falsifiable predictions. On instruction-tuned Pythia pairs, experiments yield three empirical laws: (i)~a \emph{dispersion law} -- the benefit of multi-temperature aggregation grows monotonically with the effective temperature dispersion of the views, not with their number; (ii)~\emph{dispersed views unlock the aggregation operator} -- the barycenter separates from the arithmetic mixture exactly when transport-based aggregation starts to beat averaging; and (iii)~a \emph{two-regime picture} governed by the ceiling gap $\Gamma=\PPL_{\mathrm{SFT}}-\PPL_T$: when the fine-tuned teacher barely beats a supervised student the gentle transport objective is the best KD loss but no KD beats supervised fine-tuning, whereas at a real ceiling the ranking inverts -- and the sign of the fidelity-generalization correlation flips. We argue that ``which distillation loss is the best'' is not a fixed property of the loss but a function of $\Gamma$.
\end{abstract}

\section{Introduction}
\label{sec:intro}

Knowledge distillation \citep{hinton2015distilling} trains a small
student on the outputs of a larger teacher. At every answer position $t$
both models define a categorical distribution over the vocabulary, and
the classical objective adds to the ground-truth cross-entropy a
divergence between the two softened distributions. Two structural
choices are hidden in this template. First, the teacher is summarized by
a \emph{single} temperature, although its logits define an entire curve
of distributions -- from uniform to greedy -- whose different points
expose different information (tails vs.\ modes). Second, the divergence
is \emph{pointwise}: the forward-KL gradient in the student logits is
$\tau(p^S-p^T)$, a coordinate-wise mass comparison that charges the same
price for putting mass on a synonym as on an unrelated token. Optimal
transport (OT) offers the complementary geometry: under a ground cost
derived from token embeddings, near-misses are cheap and semantic
outliers are expensive.

This paper develops and stress-tests the resulting \emph{distributional
view of KD}: represent the teacher by $K$ tempered views
$\mu_k=\sigma_{\tau_k}(z^T)$, aggregate them in the transport geometry
(a multi-temperature Wasserstein barycenter, or a Sinkhorn-divergence
objective whose optimum is that barycenter), and train the student
against the aggregate. Our contributions are:

\textbf{(1) A formal framework}
(\S\ref{sec:prelim}-\S\ref{sec:theory}) with a single notation system:
tempered views as a one-parameter exponential family; the divergence
family and its gradients; entropic OT and the debiased Sinkhorn
divergence; and the trichotomy of means. We prove that log-linear
pooling of tempered views of one teacher \emph{collapses} to a single
temperature (Prop.~\ref{prop:collapse}), so any multi-view benefit must
come from mixtures or transport barycenters, and we give a
multi-marginal Schr\"odinger-bridge (MMSB) reading
\citep{chen2019multimarginal} that predicts the operative variable of a
view set is its \emph{effective dispersion}, not its cardinality.

\textbf{(2) A controlled eight-method study}
(\S\ref{sec:setup}-\S\ref{sec:results}) on instruction-tuned Pythia
pairs \citep{biderman2023pythia} over Dolly-15k
\citep{conover2023dolly}: supervised fine-tuning (SFT), forward/
reverse/Jensen-Shannon KL, an arithmetic-mixture target, an entropic
barycenter target, a barycenter gradient-descent variant, and the
flagship \BaryOT{} (debiased Sinkhorn divergence
\citep{feydy2019interpolating} to the views, in \emph{hub} and
\emph{path} forms), all sharing one training loop, $\alpha$, data and
seeds. Methodologically we introduce a \emph{margin-aware
teacher-strength guard} and a \emph{zero-shot capacity probe} that
diagnose whether a KD experiment can demonstrate gains at all -- an
issue we found decisive and under-reported.

\textbf{(3) Three empirical laws with mechanism tests.} A monotone
\emph{dispersion law} for target-side aggregation, confirmed by a
controlled $\Delta$-sweep and a weights-vs-count probe; the
\emph{views-unlock-the-operator} effect, in which the
barycenter-mixture $\ell_1$ gap and the SinkhornBary-vs-Mixture
performance gap grow together; and a \emph{hub-vs-path} decomposition
showing that in the flagship the view-set sensitivity is a softening
artifact (path mode erases it and improves both settings).

\textbf{(4) The two-regime picture.} Across six runs and four settings,
the ceiling gap $\Gamma$ separates a thin-ceiling regime -- where no KD
beats SFT, the OT flagship minimizes the ``KD tax'', and fidelity
\emph{anti}-correlates with generalization, consistent with
\citet{stanton2021does} -- from a real-ceiling regime where the ranking
inverts, JS distillation nearly matches the teacher, and fidelity
correlates \emph{positively} with generalization. The crossover is
bracketed, $\Gamma^\star\in(0.10,1.49)$, and we state it as a
falsifiable hypothesis with concrete tests.

We report negative results with the same prominence as positive ones:
at thin ceilings every KD method pays a tax over SFT, and debiasing the
Sinkhorn divergence changes nothing measurable in this regime. We
believe this honesty is the paper's point: the distributional view is
valuable, but \emph{when} each of its ingredients pays depends on a
measurable property of the experiment.

\section{Related work}
\label{sec:related}

\textbf{Knowledge distillation.} The softened-softmax objective is due
to \citet{hinton2015distilling}. \citet{stanton2021does} showed that
student-teacher fidelity and student generalization can dissociate -- %
more faithful students do not always generalize better; our
Regime~I reproduces this inversion and Regime~II shows its sign flip.
For autoregressive language models, MiniLLM \citep{gu2024minillm}
argues for reverse KL with policy-style training and GKD
\citep{agarwal2024gkd} for on-policy data with generalized
JS objectives; both concern \emph{which} distributions are compared,
while we vary \emph{how} they are compared and aggregated. Our
off-policy reverse-KL baseline behaves exactly as this literature
predicts (weakest of the KL family, \S\ref{sec:results}).

\textbf{Computational optimal transport.} Entropic regularization and
the Sinkhorn algorithm are due to \citet{cuturi2013sinkhorn}; the
debiased Sinkhorn divergence, with positivity, convexity and
metrization of convergence in law, to \citet{feydy2019interpolating}.
Wasserstein barycenters were introduced by \citet{agueh2011barycenters};
fast computation via iterative Bregman projections by
\citet{benamou2015bregman} (the routine we use) and
\citet{cuturi2014fast}; the entropic bias of barycenters and its
correction by \citet{janati2020debiased}. The dynamical picture we
invoke originates in \citet{benamou2000cfd}; the connection between
entropic OT and Schr\"odinger bridges is surveyed from the
stochastic-control viewpoint by \citet{chen2016relation}, and the
multi-marginal bridge problem -- the closest formal object to our
multi-temperature constraint set -- by \citet{chen2019multimarginal}.

\textbf{Ensembles and diversity.} Distilling \emph{diverse} sources is
known to matter when the sources are genuinely different models
\citep{nam2021diversity}. Our views are tempered images of \emph{one}
model, which is exactly why Prop.~\ref{prop:collapse} is needed: some
aggregation operators cannot create diversity benefits from a single
teacher even in principle.

\section{Notation and preliminaries}
\label{sec:prelim}

\paragraph{Objects.} A vocabulary $\Vv$, $V=|\Vv|$, with token
embeddings $e_v\in\mathbb R^{d}$; prompt-answer pairs $x$ with answer
positions $M(x)$ (the loss mask); at $t\in M(x)$ teacher and student
logits $z^T_t,z^S_t\in\mathbb R^{V}$. The tempered softmax at
temperature $\tau>0$, with inverse temperature $\beta=1/\tau$, is
\begin{equation}
\sigma_\tau(z)_v=\frac{\exp(z_v/\tau)}{\sum_{w\in\Vv}\exp(z_w/\tau)}
\qquad\Longleftrightarrow\qquad
p_\beta\propto e^{\beta z}.
\label{eq:family}
\end{equation}
The right-hand form is a one-parameter exponential family in $\beta$: a
smooth curve in the simplex running from the uniform prior
$p_{\beta\to0}=\mathrm{Unif}(\Vv)$ to the greedy vertex
$p_{\beta\to\infty}=\delta_{\argmax z}$, with score identity
$\tfrac{d}{d\beta}\log p_\beta(v)=z_v-\mathbb E_{p_\beta}[z]$, so
raising $\beta$ monotonically sharpens the distribution toward the mode.
We call $\{p_\beta\}_{\beta>0}$ the \emph{annealing path} of the logits.

\paragraph{The KD template.} Every method in this paper is an instance
of
\begin{equation}
\mathcal L(\theta_S)=(1-\alpha)\,\mathcal L_{\mathrm{CE}}
+\alpha\,\frac{1}{|M|}\sum_{t\in M}
\mathbf D\big(\text{teacher}_t,\text{student}_t\big),
\qquad \alpha=\tfrac12,
\label{eq:template}
\end{equation}
where the first term is the masked ground-truth negative log-likelihood
and the second is a divergence between the two conditionals at position
$t$. The methods differ only in $\mathbf D$ and in which functionals of
the annealing path \eqref{eq:family} it compares.

\paragraph{The divergence family.} With $p^T=\sigma_\tau(z^T_t)$ and
$p^S=\sigma_\tau(z^S_t)$ at a shared $\tau=\tau_{\mathrm{KD}}$, and
$m=\tfrac12(p^T+p^S)$,
\begin{equation}
\KL(p^T\Vert p^S)=\sum_v p^T_v\log\frac{p^T_v}{p^S_v},
\qquad
\KL(p^S\Vert p^T),
\qquad
\JS=\tfrac12\KL(p^T\Vert m)+\tfrac12\KL(p^S\Vert m).
\label{eq:divfamily}
\end{equation}
The forward direction is \emph{mode-covering}: it takes expectations
under $p^T$, so unmatched teacher support costs $+\infty$ and the
student must cover the tails. The reverse direction is
\emph{mode-seeking}: it penalizes only where the student puts mass and
may drop teacher modes. The Jensen-Shannon divergence is bounded and
symmetric. With the conventional $\tau^2$ scaling, the forward gradient
is the pointwise
\begin{equation}
\nabla_{z^S}\,\tau^2\,\KL(p^T\Vert p^S)=\tau\,(p^S-p^T),
\label{eq:klgrad}
\end{equation}
which is exactly the blindness the transport machinery repairs: it
charges the same for misplacing mass on a synonym as on an arbitrary
token.

\paragraph{Support restriction.} All geometric methods act on the
shared support $\Ss_t=\text{top-}k(\bar p_t)\subset\Vv$ of the
position-averaged teacher probabilities and \emph{condition} each
distribution on it, $\tilde p=p|_{\Ss_t}/p(\Ss_t)$. All divergences
below are therefore divergences between conditional laws on $\Ss_t$; the
gain is a $k\times k$ instead of $V\times V$ cost matrix, and ablations
over $k\in\{64,256\}$ show no benefit from the larger support here
($11.40$ vs.\ $11.46$; Table~\ref{tab:abl}).

\section{The BaryKD family}
\label{sec:methods}

\paragraph{Ground cost.} On $\Ss_t$ we measure semantic dissimilarity by
normalized embedding cosine distance,
\begin{equation}
C_{ij}=\frac{1-\cos(e_i,e_j)}
            {\max_{i',j'}\big(1-\cos(e_{i'},e_{j'})\big)}\in[0,1],
\qquad C=C^\top,\quad C_{ii}=0.
\label{eq:cost}
\end{equation}
$C$ is symmetric, bounded and vanishes on the diagonal, which is all
the theory below requires; we do not claim the triangle inequality.

\paragraph{Entropic OT and the debiased Sinkhorn divergence.} For
$a,b\in\Delta_{\Ss_t}$,
\begin{equation}
\OT_\eps(a,b)=\min_{P\in U(a,b)}\ \langle P,C\rangle
+\eps\,\KL\!\big(P\,\Vert\,ab^\top\big),
\qquad
U(a,b)=\{P\ge0:P\mathbf 1=a,\ P^\top\mathbf1=b\},
\label{eq:eot}
\end{equation}
where $P_{ij}$ is the mass moved from $i$ to $j$ and $C_{ij}$ its unit
price; the entropic term is strictly convex and makes the problem
GPU-fast, solved by Sinkhorn iterations \citep{cuturi2013sinkhorn}. The
smoothing induces the \emph{entropic blur}: $\OT_\eps(a,a)>0$ and the
minimizer of $\OT_\eps(a,\cdot)$ is a smoothed $a$. The Sinkhorn
divergence \citep{feydy2019interpolating} subtracts the two self-blur
terms,
\begin{equation}
S_\eps(a,b)=\OT_\eps(a,b)-\tfrac12\OT_\eps(a,a)-\tfrac12\OT_\eps(b,b),
\label{eq:sdiv}
\end{equation}
and is nonnegative, convex, vanishes iff $a=b$, and metrizes
convergence in law. In our objectives $a$ is the student and $b$ the
teacher, so $\OT_\eps(b,b)$ is constant in the optimization and is
omitted from the computational graph.

\paragraph{Views and aggregation targets.} The teacher's
multi-temperature views are
$\mu_k=\widetilde{\sigma_{\tau_k}(z^T_t)}$, $k=1,\dots,K$, with weights
$\lambda\in\Delta_K$, uniform unless stated. Two view sets recur: the
\emph{endpoints} $\{0.5,1.5\}$ ($K{=}2$, the default) and the
\emph{interior} set $\{0.7,1.0,1.3\}$ ($K{=}3$). The methods:
\begin{itemize}\itemsep2pt
\item \textbf{SFT}: pure cross-entropy at equal compute (2 epochs),
i.e.\ $\alpha=0$ in \eqref{eq:template}.
\item \textbf{VanillaKD / ReverseKD / JSKD}: $\mathbf D$ from
\eqref{eq:divfamily} at $\tau_{\mathrm{KD}}=2$.
\item \textbf{MixtureKD}: forward KL to the arithmetic mixture
$\bar\mu=\sum_k\lambda_k\mu_k$.
\item \textbf{SinkhornBary}: forward KL to the precomputed entropic
barycenter, the Fr\'echet mean of the views in the transport geometry,
\begin{equation}
b^\star=\argmin_{\nu\in\Delta_{\Ss_t}}\ \sum_{k}\lambda_k\,
\OT_\eps(\mu_k,\nu),
\label{eq:bary}
\end{equation}
computed by iterative Bregman projections \citep{benamou2015bregman};
existence and uniqueness of the unregularized $W_2$ barycenter under
absolute continuity are due to \citet{agueh2011barycenters}, and the
entropic bias of \eqref{eq:bary} and its correction to
\citet{janati2020debiased}.
\item \textbf{BARYGD}: an unrolled gradient-descent variant of
\eqref{eq:bary} (15 steps), retained as a stress test of the
fixed-point solver.
\item \textbf{\BaryOT{} (flagship)}: no precomputed target; the student is
pulled to the views directly through $S_\eps$, in one of two forms
(\S\ref{sec:theory}).
\end{itemize}
Shared hyperparameters: $\eps=0.05$ on the $[0,1]$-bounded cost,
$k=256$, $\tau_{\mathrm{KD}}=2$, $\alpha=\tfrac12$; $200$ Bregman
iterations in the barycenter solver \eqref{eq:bary} and $50$ Sinkhorn
iterations inside each $S_\eps$ evaluation \eqref{eq:sdiv}.

\paragraph{Implementation note.} The KL-family losses
\eqref{eq:divfamily} are evaluated on the full vocabulary at every
masked position and carry the conventional $\tau^2$ scaling of
\eqref{eq:klgrad}. The four geometric methods are evaluated on the
top-$k$ support at $4$ positions sampled uniformly per sequence -- a
Monte-Carlo estimator of the position average in \eqref{eq:template},
used because each position costs a Sinkhorn solve -- and carry no
$\tau^2$ factor. All methods otherwise share the training loop, data,
$\alpha$ and seeds.

\section{Theory: three means, two objectives, one bridge}
\label{sec:theory}

\paragraph{The trichotomy of means.} Given tempered views of \emph{one}
teacher, there are three natural averages: log-linear (geometric),
arithmetic (mixture), and transport (barycenter). The first collapses:

\begin{proposition}[Geometric collapse]
\label{prop:collapse}
Let $\mu_k=p_{\beta_k}\propto e^{\beta_k z}$ as in \eqref{eq:family} and
$\lambda\in\Delta_K$. Then the normalized geometric mean is again a
tempered view:
\[
\frac{\prod_k\mu_k^{\lambda_k}}{\big\lVert\prod_k\mu_k^{\lambda_k}\big\rVert_1}
= p_{\bar\beta},
\qquad \bar\beta=\textstyle\sum_k\lambda_k\beta_k .
\]
\end{proposition}
\begin{proof}
$\prod_k\mu_k^{\lambda_k}(v)\propto\prod_k e^{\lambda_k\beta_k z_v}
= e^{(\sum_k\lambda_k\beta_k)z_v}$; normalizing yields
$p_{\bar\beta}$.
\end{proof}

\noindent Hence log-linear pooling of tempered views is
\emph{equivalent to choosing a single temperature}: it cannot be the
source of a multi-view benefit, and any observed benefit must be
attributed to the mixture or the barycenter.

\begin{remark}[The mixture leaves the family]
\label{rem:mixture}
$\bar\mu=\sum_k\lambda_k\mu_k$ is in general \emph{not} of the form
$e^{\beta z}/Z$: the curve $\beta\mapsto p_\beta$ is not a line segment
in the simplex whenever $z$ takes at least three distinct values, so a
convex combination of two distinct points on it lies off the curve. The
mixture preserves the $\argmax$ but thickens tails -- it adds mass
without consulting the token geometry. The barycenter \eqref{eq:bary}
is the third mean: it averages \emph{in the geometry induced by} $C$.
\end{remark}

\paragraph{Hub vs.\ path objectives.} With student views
$p^S_\tau=\widetilde{\sigma_\tau(z^S_t)}$, the flagship admits two
forms:
\begin{equation}
\mathcal L_{\mathrm{hub}}
=\sum_k\lambda_k\,S_\eps\big(p^S_{\tau_{\mathrm{KD}}},\mu_k\big),
\qquad\qquad
\mathcal L_{\mathrm{path}}
=\sum_k\lambda_k\,S_\eps\big(p^S_{\tau_k},\mu_k\big).
\label{eq:hubpath}
\end{equation}
In \emph{hub} mode one student view meets all teacher views, and
first-order optimality places $p^S_{\tau_{\mathrm{KD}}}$ at the
$S_\eps$-barycenter of $\{\mu_k\}$: the student is an \emph{implicit
barycenter variable} and nothing is precomputed. In \emph{path} mode
the student's own tempered curve must track the teacher's at each
sampled $\beta_k$ -- a discrete two-curve matching, coupled through the
shared logits $z^S_t$. Hub entangles two asymmetries, the comparison
geometry \emph{and} a fixed softening $\tau_{\mathrm{KD}}$ applied to
one side only; path removes the second. This yields a clean mechanism
test (\S\ref{sec:results}): if a view-set effect in hub mode disappears
in path mode, it was a softening artifact, not geometry.

\paragraph{The multi-marginal bridge reading (motivating, not
literal).} The multi-marginal Schr\"odinger bridge problem
\citep{chen2019multimarginal,chen2016relation} seeks the path law
$\mathbb P$ closest to a reference law $\mathbb R$ subject to
prescribed marginals at multiple times:
\begin{equation}
\min_{\mathbb P}\ \KL\big(\mathbb P\,\Vert\,\mathbb R\big)
\quad\text{subject to}\quad
\mathbb P_{t_k}=\rho_{t_k},\quad k=1,\dots,K.
\label{eq:mmsb}
\end{equation}
Reading inverse temperature as time ($t\leftrightarrow\beta$), the
teacher's annealing path supplies the marginals and the student is the
steered object. We use \eqref{eq:mmsb} only as a surrogate motivation
and derive from it two \emph{falsifiable} predictions: (i) endpoint
marginals are the binding constraints of a bridge, so interior views
should contribute little beyond re-weighting the objective toward the
middle; (ii) the operative variable of a view set
$(\lambda,\tau)$ is its \emph{effective dispersion}
\begin{equation}
\Deff(\lambda,\tau)=\sum_k\lambda_k\,\lvert\tau_k-1\rvert ,
\label{eq:deff}
\end{equation}
not the number of views $K$.

\paragraph{The ceiling gap and the two regimes.} Define
$\Gamma:=\PPL_{\mathrm{SFT}}-\PPL_T$ on held-out data, positive iff the
fine-tuned teacher is a real ceiling for a supervised student at equal
compute. We will show the data separate a \emph{thin-ceiling} regime
($\Gamma\lesssim0.1$) from a \emph{real-ceiling} regime
($\Gamma\approx1.5$) with inverted method rankings, and state:

\begin{hypothesis}[Crossover]
\label{hyp:crossover}
There exists $\Gamma^\star$ such that for $\Gamma<\Gamma^\star$ no KD
method beats SFT and the gentle transport objective minimizes the KD
tax, while for $\Gamma>\Gamma^\star$ faithful KL-family KD beats SFT
and the ranking inverts. Our runs bracket
$\Gamma^\star\in(0.10,\,1.49)$.
\end{hypothesis}

\section{Experimental setup}
\label{sec:setup}

\textbf{Models and data.} Pythia pairs \citep{biderman2023pythia}:
\emph{small} $160\text{M}\to31\text{M}$ and \emph{big}
$410\text{M}\to70\text{M}$ ($5.3\times$ and $5.8\times$
compression by parameter count), all in fp32. Dolly-15k \citep{conover2023dolly} with the
\texttt{general\_qa} category held out entirely; 3{,}000 training
examples, 200 held-out evaluation examples, max length 256. The teacher
is fine-tuned 3 epochs (AdamW, $5\cdot10^{-5}$); students train 2
epochs with identical loops, $\alpha=\tfrac12$, 5 seeds.

\textbf{Metrics.} Held-out perplexity
$\PPL=\exp\big(\frac1{|M|}\sum_{t\in M}-\log p^S_1(y_t)\big)$
measures generalization, where $p^S_1$ denotes the student's
\emph{untempered} ($\tau=1$) conditional. Fidelity is measured by
$\mathrm{agree@1}$, the teacher-student top-1 agreement rate, and
$\mathrm{predKL}=\frac1{|M|}\sum_t\KL(p^T_1\Vert p^S_1)$, both after
\citet{stanton2021does}; calibration by 15-bin expected calibration
error (ECE) \citep{guo2017calibration}. We report seed-paired
$t$-tests, Cohen's $d$, and bootstrap 95\% confidence intervals.

\textbf{Guard and probe.} Before any KD run we report (i) the
\emph{zero-shot capacity probe} -- held-out PPL of the \emph{pretrained}
teacher and student, diagnosing whether capacity is the lever at
all -- and (ii) the \emph{margin-aware guard}, which prints $\Gamma$
with levels \textsc{fail}/\textsc{thin}/\textsc{ok}. On the small pair
the pretrained gap is huge ($328.6$ vs $712.6$) yet 3k-shot fine-tuning
compresses it to $\Gamma=+0.10$ (\textsc{thin}): the task
\emph{saturates}. On the big pair $\Gamma=+1.49$ (\textsc{ok}). This
diagnostic distinction turned out to organize every downstream result.

\textbf{Settings.} Across the project the same pipeline was executed
under four teacher-strength settings, referred to throughout -- and used
as the point labels of Fig.~\ref{fig:meta} -- by what distinguishes
them: \emph{under-trained teacher} (one fine-tuning epoch, 1k training
examples, fp16 evaluation); \emph{longer teacher FT} (three epochs,
otherwise identical); \emph{saturated task} (three epochs, 3{,}000
examples, fp32 evaluation -- executed three times: two independent
replications, and a third whose flagship uses the endpoint views); and
\emph{capacity ceiling} (the $410\text{M}\to70\text{M}$ pair of
Table~\ref{tab:main}). The first two are retained as \emph{diagnostic}
settings: they document how an apparently reasonable KD setup silently
places the teacher below the student -- a failure no method-level
comparison can survive.

\section{Results}
\label{sec:results}

\begin{table}[t]
\caption{\textbf{Main results} (mean $\pm$ s.d.\ over 5 seeds; held-out
\texttt{general\_qa}). Left block: thin ceiling ($\Gamma=+0.10$);
right block: real ceiling ($\Gamma=+1.49$). \textbf{Bold} = best
\emph{KD} method within a block and column; \underline{underline} =
beats SFT on PPL ($p<.05$, seed-paired). SFT retains the best ECE in
both blocks; the bold ECE is the best among KD methods. The flagship
row uses the endpoint views $\{0.5,1.5\}$ in hub mode; both choices are
varied in Table~\ref{tab:abl}.}
\label{tab:main}
\centering
\small
\setlength{\tabcolsep}{2.5pt}
\begin{tabular}{@{}l cccc c cccc@{}}
\toprule
& \multicolumn{4}{c}{\textbf{small}\, $160\text{M}\!\to\!31\text{M}$, $\Gamma{=}{+}0.10$} &
& \multicolumn{4}{c}{\textbf{big}\, $410\text{M}\!\to\!70\text{M}$, $\Gamma{=}{+}1.49$}\\
\cmidrule(r){2-5}\cmidrule(l){7-10}
method & PPL\,$\downarrow$ & agree@1\,$\uparrow$ & predKL\,$\downarrow$ & ECE\,$\downarrow$ &
       & PPL\,$\downarrow$ & agree@1\,$\uparrow$ & predKL\,$\downarrow$ & ECE\,$\downarrow$ \\
\midrule
\multicolumn{10}{@{}l}{\emph{Reference models}}\\
Teacher (3-ep.\ FT) & 10.664 & -- & -- & .123 & & 9.602 & -- & -- & .195\\
SFT (no KD)          & 10.762\sd{.017} & .528 & 1.158 & .060 &
                     & 11.089\sd{.357} & .468 & 1.587 & .068\\
\addlinespace
\multicolumn{10}{@{}l}{\emph{Pointwise divergences: the KL family}}\\
VanillaKD & 11.923\sd{.024} & \textbf{.550} & \textbf{0.984} & .087 &
          & \underline{10.588}\sd{.049} & \textbf{.488} & 1.429 & .153\\
ReverseKD & 12.422\sd{.019} & .541 & 1.044 & .095 &
          & 11.158\sd{.134} & .484 & 1.651 & .135\\
JSKD      & 11.535\sd{.026} & .547 & 1.018 & \textbf{.071} &
          & \underline{\textbf{9.888}}\sd{.027} & .484 & \textbf{1.428} & \textbf{.094}\\
\addlinespace
\multicolumn{10}{@{}l}{\emph{Aggregated targets}}\\
MixtureKD    & 12.719\sd{.066} & .510 & 1.411 & .177 &
             & 13.781\sd{.097} & .427 & 2.007 & .171\\
SinkhornBary & 11.424\sd{.023} & .520 & 1.213 & .107 &
             & 12.787\sd{.077} & .431 & 1.881 & .136\\
BARYGD       & 13.112\sd{.109} & .505 & 1.466 & .174 &
             & 14.182\sd{.143} & .418 & 2.038 & .153\\
\addlinespace
\multicolumn{10}{@{}l}{\emph{Flagship: Sinkhorn divergence to the views}}\\
\BaryOT{} (hub) & \textbf{11.117}\sd{.023} & .530 & 1.200 & .101 &
             & 11.548\sd{.097} & .462 & 1.682 & .111\\
\bottomrule
\end{tabular}
\end{table}

\begin{table}[t]
\caption{\textbf{Ablations} (small pair, $\Gamma=+0.10$; held-out PPL).
The view-set column uses SinkhornBary as the sensitive geometric probe
and $\Deff$ is the effective dispersion \eqref{eq:deff}; the flagship
column varies \BaryOT{}. Bold marks the best setting within each
ablation. Seed counts differ by block -- 3 for the dispersion sweep, 2
for the $K{=}3$ probes, the flagship variants and the debiasing pair, 5
for the reference rows and for endpoints-hub -- which is why these
numbers carry small offsets from the 5-seed Table~\ref{tab:main}.
Uniform $K{=}3$ sits $\approx0.8$ PPL \emph{above} the $\Deff$ curve
traced by the dispersion sweep (Fig.~\ref{fig:dispersion}a), whereas
down-weighting the midpoint lands on it.}
\label{tab:abl}
\centering
\small
\setlength{\tabcolsep}{4pt}
\begin{tabular}{@{}lcc@{\hspace{2.6em}}lc@{}}
\toprule
\multicolumn{3}{@{}l}{\textbf{View set} (target-side aggregation)} &
\multicolumn{2}{l}{\textbf{Flagship objective and implementation}}\\
\cmidrule(r){1-3}\cmidrule(l){4-5}
setting & $\Deff$ & PPL\,$\downarrow$ & setting & PPL\,$\downarrow$\\
\midrule
\multicolumn{3}{@{}l}{\emph{Dispersion sweep} ($\tau=1\pm\Delta$; $\Delta{=}0$ is one view)} &
\multicolumn{2}{l}{\emph{Views $\times$ student-view mode}}\\
$\Delta{=}0$   & 0    & 13.08\sd{.06} & interior-hub   & 11.28\\
$\Delta{=}0.1$ & 0.10 & 12.54\sd{.04} & endpoints-hub  & 11.12\\
$\Delta{=}0.3$ & 0.30 & 11.90\sd{.03} & interior-path  & \textbf{10.88}\\
$\Delta{=}0.5$ & 0.50 & \textbf{11.47}\sd{.01} & endpoints-path & 10.89\\
\addlinespace
\multicolumn{3}{@{}l}{\emph{Weights vs.\ count} ($K{=}3$ on $\{0.5,1,1.5\}$)} &
\multicolumn{2}{l}{\emph{Support size}}\\
uniform                   & 0.33 & 12.62\sd{.07} & top-$k$ $=64$  & \textbf{11.40}\\
midpoint $\lambda{=}0.10$ & 0.45 & \textbf{11.54} & top-$k$ $=256$ & 11.46\\
\addlinespace
\multicolumn{3}{@{}l}{\emph{Reference}} &
\multicolumn{2}{l}{\emph{Sinkhorn debiasing}}\\
SFT (no KD)         & -- & 10.76\sd{.02} & debiased & \textbf{11.13}\\
Teacher (3-ep.\ FT) & -- & 10.66         & biased   & 11.14\\
\bottomrule
\end{tabular}
\end{table}

\begin{figure}[t]
\centering
\includegraphics[width=\linewidth]{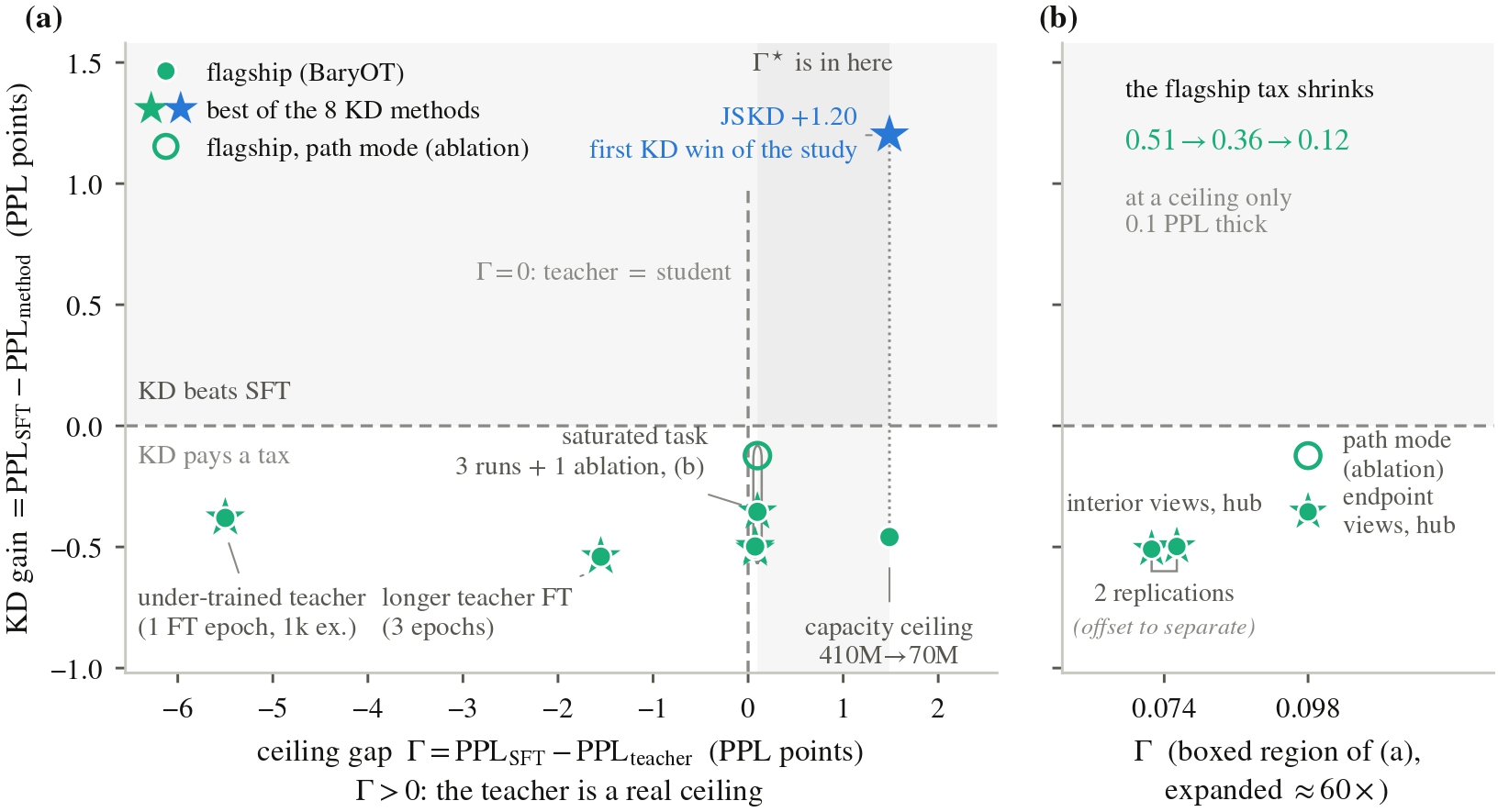}
\caption{\textbf{One point per run, grouped by experimental setting:
what varies along the $x$-axis is the teacher, not the method.} All settings share the
dataset, the training loop and the eight methods of
Table~\ref{tab:main}; they differ in how strong the fine-tuned teacher
ends up relative to an SFT student (the ceiling gap $\Gamma$) and,
within the saturated settings, in the flagship's view configuration.
Left to right: an \emph{under-trained teacher} (one fine-tuning epoch
on 1k examples, fp16 evaluation) ends up $5.5$ PPL \emph{worse} than
its own student, so every faithful KD transfers weakness; \emph{longer teacher
fine-tuning} (three epochs, same data) closes most of the deficit but
the teacher still trails; the \emph{saturated-task} setting (three
epochs, 3k examples, fp32 evaluation) brings the two within a tenth of
a perplexity point ($\Gamma=0.07$-$0.10$) -- the pretrained capacity gap
is $+384$ PPL, but 3k-shot fine-tuning saturates the task -- shown as
two independent replications plus a third run whose flagship uses the
endpoint views $\{0.5,1.5\}$; the \emph{capacity-ceiling} setting
swaps in the $410\text{M}\to70\text{M}$ pair, which preserves
$\Gamma=+1.49$ after identical fine-tuning. Circles: flagship gain
(SFT $-$ \BaryOT{}); stars: best-KD gain; open triangle: the
path-matching flagship variant of Eq.~\eqref{eq:hubpath}. Vertical
differences at the same $\Gamma$ therefore isolate \emph{method}
effects (endpoint views $-0.15$; path mode a further $-0.23$);
horizontal movement isolates the \emph{teacher} effect. Wherever
$\Gamma\lesssim0.1$ no method beats SFT, and across the saturated
settings the flagship pays a $0.35$-$0.51$ tax; at $\Gamma=+1.49$ the best-KD star enters the
winning half-plane -- by JSKD ($+1.20$), not the transport flagship
($-0.46$). Hypothesis~\ref{hyp:crossover} brackets the crossover
between the last two settings.}
\label{fig:meta}
\end{figure}

\begin{figure}[t]
\centering
\includegraphics[width=\linewidth]{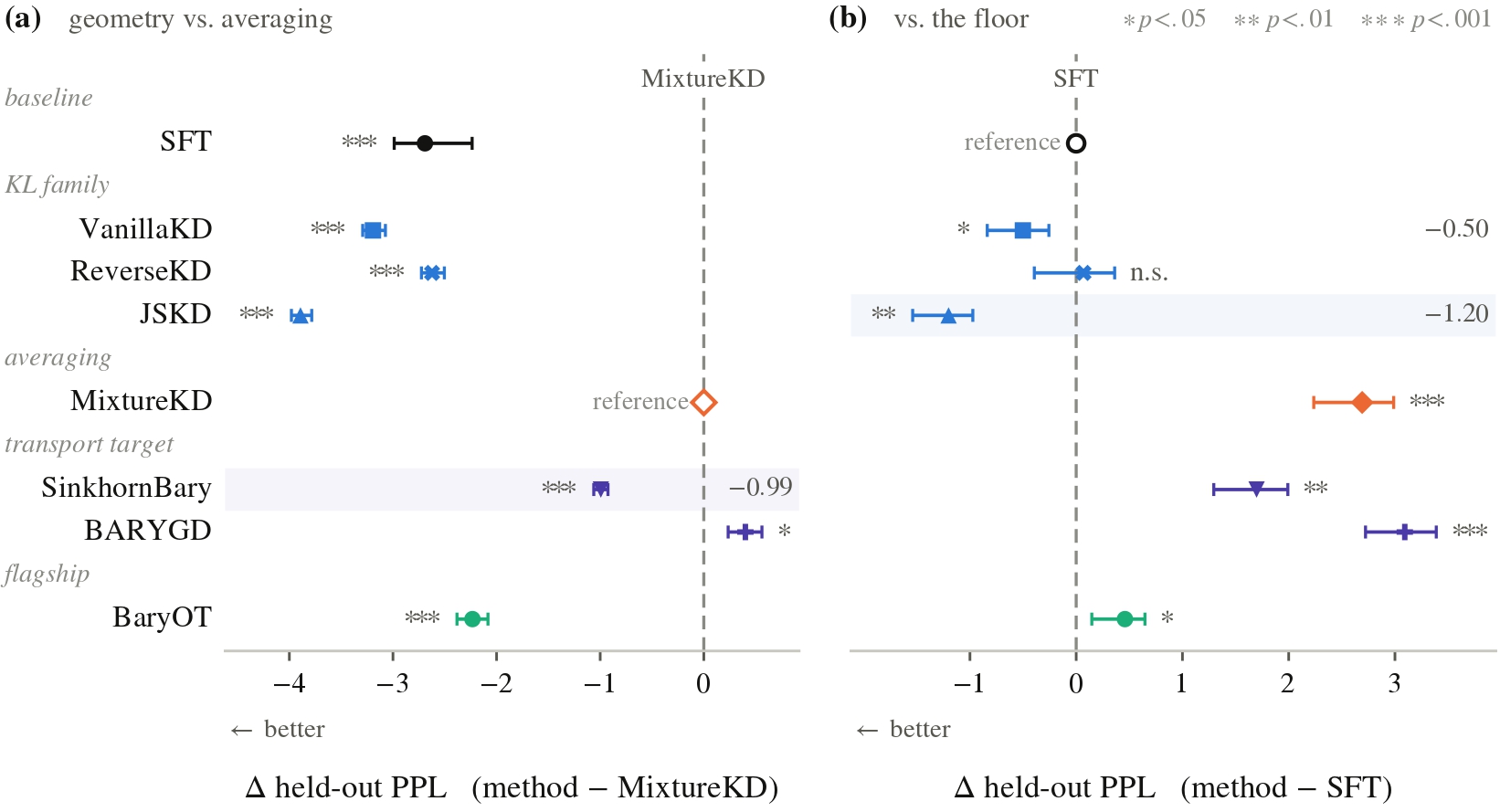}
\caption{\textbf{Real-ceiling run ($410\text{M}\to70\text{M}$): paired
differences with bootstrap 95\% CIs.} Right panel: JSKD ($-1.20$,
$p=.0023$) and VanillaKD ($-0.50$, $p=.043$) beat SFT -- the first
KD wins of the study -- while the averaging and transport families fall
behind. Left panel: the transport barycenter still beats the arithmetic
mixture by $-0.99$ ($p<10^{-4}$): the operator effect survives the
regime flip.}
\label{fig:forest}
\end{figure}

\begin{figure}[t]
\centering
\includegraphics[width=\linewidth]{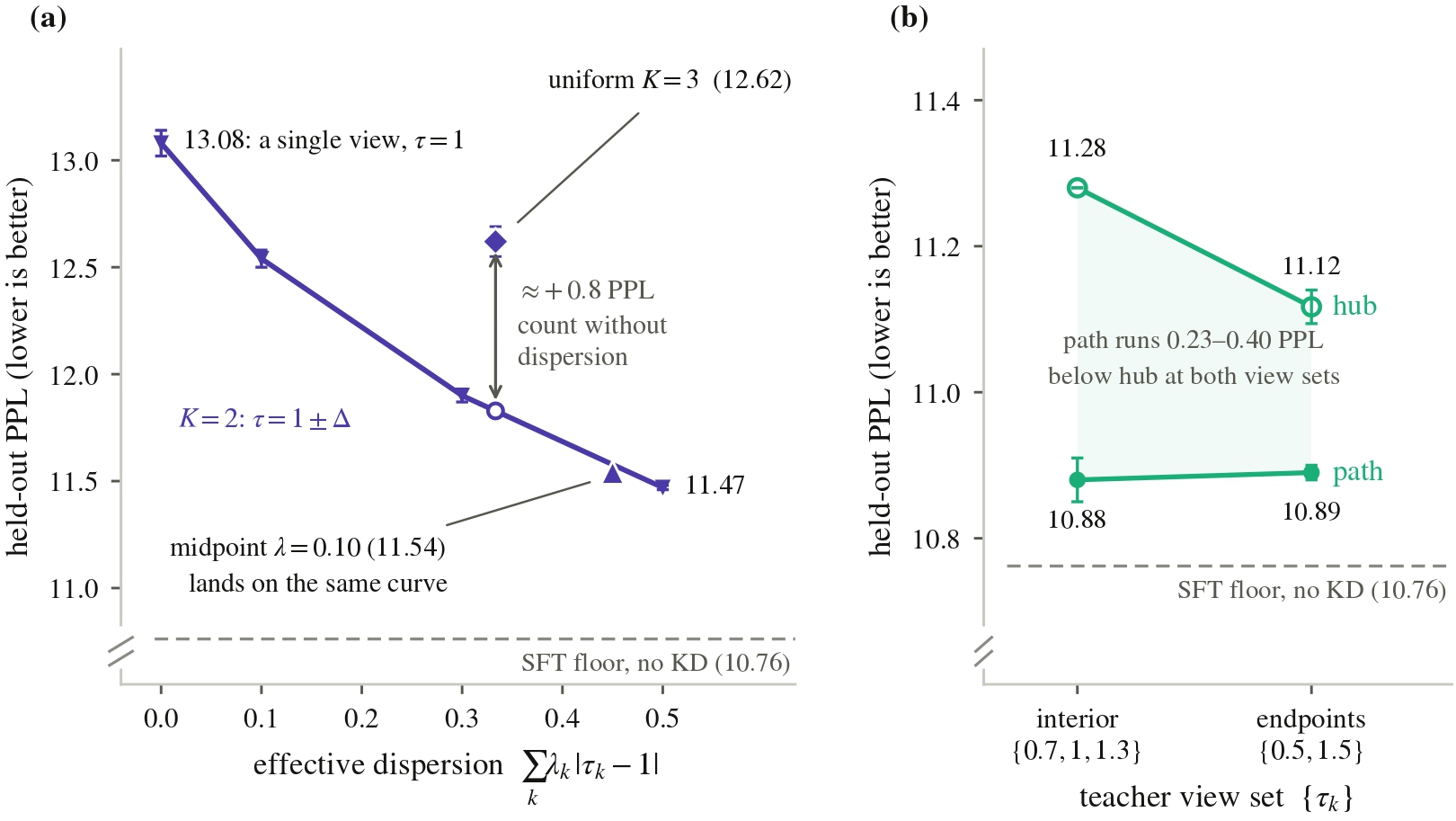}
\caption{\textbf{The dispersion law} (small pair, SinkhornBary target in
(a), \BaryOT{} in (b)). \textbf{(a)}~Held-out PPL falls monotonically as
the two views move apart symmetrically ($\tau=1\pm\Delta$; $\Delta{=}0$
is the single-view control), and the two $K{=}3$ probes obey the same
curve through their effective dispersion, not their count: the
midpoint-down-weighted probe lands on it, uniform $K{=}3$ sits
$\approx0.8$ PPL above. The open marker is the measured sweep
\emph{linearly interpolated} to the uniform probe's $\Deff=1/3$; the
quoted gap is that interpolation, not a separate run.
\textbf{(b)}~In the flagship, hub mode is view-sensitive while path mode
is flat and uniformly better -- the view-set effect in hub mode was a
softening artifact. Error bars are seed s.d.; $n{=}3$ for the sweep,
$n{=}2$ for the probes and flagship variants, $n{=}5$ for
endpoints-hub. Both panels share the SFT floor and neither $y$ axis
starts at zero.}
\label{fig:dispersion}
\end{figure}

\begin{figure}[t]
\centering
\includegraphics[width=\linewidth]{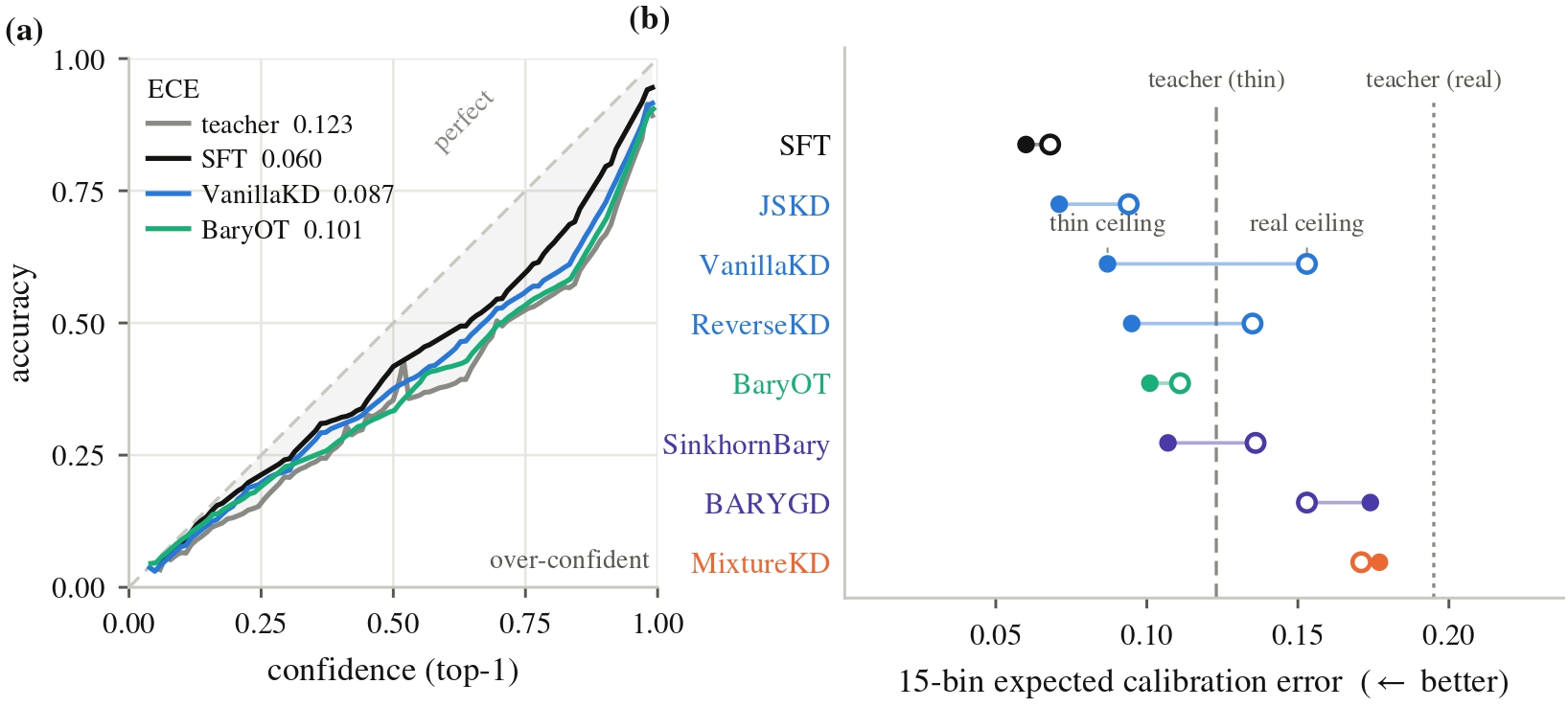}
\caption{\textbf{Calibration.} \textbf{(a)}~Next-token top-1 reliability
(small pair): every curve runs under the diagonal, so every model is
over-confident, and SFT is the closest to it. \textbf{(b)}~15-bin ECE
for all eight students in \emph{both} regimes, against each regime's
teacher (filled = thin ceiling, hollow = real ceiling). The averaging
family is the only one that fails to beat the teacher; \BaryOT{} beats
it in both regimes; and the ordering is essentially unchanged by the
regime flip, unlike the PPL ranking of Table~\ref{tab:main}.}
\label{fig:calibration}
\end{figure}

\subsection{Findings}
\label{sec:findings}

\textbf{F1  --  The ceiling problem is real, diagnosable, and was
capacity.} On the small pair the pretrained teacher-student gap is
$+384$ PPL, yet after identical 3k-example fine-tuning
$\Gamma=+0.10$: the task saturates, and \emph{no} KD method can beat
SFT there (Table~\ref{tab:main}, left). The capacity step
($410\text{M}\to70\text{M}$) preserves $\Gamma=+1.49$ after
fine-tuning. We suggest reporting a guard/probe pair of this kind as
standard practice in KD studies: without it, ``KD fails'' and ``there
was nothing to distill'' are indistinguishable.

\textbf{F2  --  The two regimes and the inversion (central result).} At
$\Gamma=+1.49$ KD beats SFT for the first time in the study:
JSKD reaches $9.888\pm.027$ -- within $0.29$ of the $5.8\times$ larger
teacher ($d=-3.1$, $p=.0023$) -- and VanillaKD $10.588$
($p=.043$), while the transport/averaging family drops to the bottom.
The ranking is the mirror image of the thin-ceiling one
(Fig.~\ref{fig:forest}, Table~\ref{tab:main}).

\textbf{F3  --  The fidelity-generalization correlation flips sign with
$\Gamma$.} At $\Gamma\approx0$, the most faithful method (VanillaKD,
predKL $0.984$) is \emph{out-generalized} by the less faithful JSKD
(predKL $1.018$, PPL $11.535$ vs.\ $11.923$) -- the
inversion of \citet{stanton2021does}. At $\Gamma=1.49$ the most
faithful methods (VanillaKD/JSKD, predKL $1.43$) are exactly the best
generalizers. Fidelity helps iff there is something to be faithful
\emph{to}.

\textbf{F4  --  The dispersion law (target pathway), with an informative
deviation.} PPL is monotone in $\Delta$ ($13.08\to11.47$,
Fig.~\ref{fig:dispersion}a); an ablation forecast of the endpoint
configuration, made in the second saturated replication ($11.44$), was
realized in the third replication's main sweep to within $0.02$
(SinkhornBary $12.78\to11.424$). Down-weighting the midpoint of
$\{0.5,1,1.5\}$ to $\lambda=(.45,.10,.45)$ lands on the
$\Deff$ curve, but \emph{uniform} $K{=}3$ sits $\approx0.8$ PPL above it:
interior views harm \emph{superlinearly} in their weight, favoring the
constraint (endpoint-binding) reading of \eqref{eq:mmsb} over pure
dilution.

\textbf{F5  --  Path beats hub and erases view sensitivity (flagship
pathway).} Path mode: $10.88/10.89$ (interior/endpoints) vs.\ hub
$11.28/11.12$ (Table~\ref{tab:abl}). The hub-mode endpoint advantage
was therefore largely the fixed-$\tau_{\mathrm{KD}}$ softening
asymmetry of \eqref{eq:hubpath}; path-matching is simply the better
objective and brings the flagship within $0.12$ PPL of SFT at a
$0.10$-thick ceiling.

\textbf{F6  --  Dispersed views unlock the operator.} The
barycenter-mixture distance grows with dispersion
($\lVert b^\star-\bar\mu\rVert_1$: $0.027$ interior $\to$ $0.424$
endpoints, small pair; $0.191$ big pair), and exactly there
SinkhornBary beats MixtureKD decisively in \emph{both} regimes
($-1.30$ and $-0.99$, $p<10^{-4}$). With clustered views the two
aggregates coincide and the operator is second-order -- resolving an
apparent contradiction in our own earlier runs.

\textbf{F7  --  The KD tax is two-sided and method-dependent.} At thin
ceilings every KD pays a tax over SFT; the flagship's shrank with
better objectives at fixed $\alpha$
($0.51\to0.36\to0.12$: views, then path), so it is not a fixed
price of the KD budget. At the real ceiling the ``tax'' turns negative for the
forward-KL and JS members (ReverseKD still pays $+0.07$, n.s.). Fig.~\ref{fig:meta} shows both series across all six runs.

\textbf{F8  --  Stable side-channels.} Calibration ordering is
regime-stable (SFT best; \BaryOT{} beats the teacher's ECE in all
runs; averaging worst; Fig.~\ref{fig:calibration}). The KL-direction
ordering $\JS<$ forward $<$ reverse holds in both regimes, with
off-policy reverse KL the only KL member that never beats
SFT -- consistent with the on-policy analyses of
\citet{gu2024minillm,agarwal2024gkd}. Debiasing \eqref{eq:sdiv} is a
measurable no-op in this regime ($11.13$ vs.\ $11.14$; four runs), as
the transport term dominates at $\eps=0.05$ on a bounded cost.
The two saturated-task replications agree to within $0.03$ PPL on
every method.

\section{Discussion and limitations}
\label{sec:discussion}

\textbf{Interpretation.} The transport objective's defining property is
\emph{tolerance}: under \eqref{eq:cost} a near-miss is cheap. When the
teacher barely exceeds the student (Regime~I), that tolerance is
protective -- the student is not dragged toward teacher noise, and the
gentle pull yields the smallest tax, best-in-family PPL, and
better-than-teacher calibration. When the teacher is genuinely stronger
(Regime~II), the same tolerance discards usable gradient signal exactly
where the near-misses carry information, and faithful pointwise
matching wins. This resolves the apparent conflict between our Regime-I
results and the classical KD successes, and refines the fidelity
paradox of \citet{stanton2021does} into a conditional statement.

\textbf{Limitations.} Regime~II rests on one model pair, one dataset
and 5 seeds (its SFT variance, $\pm0.357$, is the largest in the
study); the big-pair ablations (dispersion sweep in path mode,
\BaryOT{}-path) did not finish and are the single most valuable
missing cells; the crossover of Hypothesis~\ref{hyp:crossover} is
bracketed by two points only; the MMSB reading is a surrogate -- none of
its statements are used as theorems; models are small
($31\text{M}$-$410\text{M}$) and the data regime is 3k examples; and
the dispersion law is established for the barycenter-target pathway at
one $\eps$, while in path mode the flagship is view-insensitive,
limiting the law's practical scope to target-side aggregation. Compute
was a single T4 GPU per run.

\textbf{Future work.} (i)~A $\Gamma$-sweep at fixed architecture
(intermediate Pythia sizes, or varying teacher fine-tuning) to locate
$\Gamma^\star$; (ii)~\BaryOT{}-path at $\Gamma=1.49$, which
decides whether the two regimes are about objectives (hub/target) or
about the OT coupling itself; (iii)~a hybrid
$\mathbf D=\JS+\gamma S_\eps$ in Regime~II -- does geometric
regularization help the winner (prediction: small $\gamma$ improves ECE
at equal PPL); (iv)~on-policy variants \citep{agarwal2024gkd} of the
reverse direction.

\section{Conclusion}
\label{sec:conclusion}

We formalized multi-temperature, transport-geometric distillation from
first principles, proved which aggregation operators can and cannot
create multi-view benefits, and subjected the framework to a controlled
eight-method study with explicit teacher-strength diagnostics. The
framework's ingredients each earned a precise scope: endpoint views for
target-side aggregation (a monotone dispersion law), path-matching for
the flagship, and the transport objective as the robust choice when the
teacher has little to teach. The overarching lesson is that ``which
distillation loss is best'' is not a property of the loss: it is a
function of the ceiling gap $\Gamma$, with a bracketed crossover that
we state as a falsifiable hypothesis.

\bibliography{references}

\end{document}